\documentclass[letterpaper, 10 pt, conference]{ieeeconf}
\IEEEoverridecommandlockouts                         
\title{\LARGE \bf 
Incremental Bayesian Learning for Fail-Operational Control in Autonomous Driving 
}
  
\author{Lei Zheng, Rui Yang,  Zengqi Peng, Wei Yan, Michael Yu Wang, \textit{Fellow, IEEE,}  and Jun Ma
\thanks{This work was supported in part by the National Natural Science Foundation of China under Grant 62303390; and in part by the Project of Hetao Shenzhen-Hong Kong Science and Technology Innovation Cooperation Zone under Grant HZQB-KCZYB-2020083. \textit{(Corresponding Author: Jun Ma.)}}
    \thanks{Lei Zheng, Rui Yang, Zengqi Peng, and Wei Yan are with the Robotics and Autonomous Systems Thrust, The Hong Kong University of Science and Technology (Guangzhou), Guangzhou, China (email: lzheng135@connect.ust.hk; ryang253@connect.hkust-gz.edu.cn; zpeng940@connect.ust.hk; wyan993@connect.hkust-gz.edu.cn).}
    \thanks{Michael Yu Wang is with the School of Engineering, Great Bay University, Dongguan, China (email: mywang@gbu.edu.cn). }
\thanks{Jun Ma is with the Robotics and Autonomous Systems Thrust, The Hong Kong University of Science and Technology (Guangzhou), Guangzhou, China, also with the Division of Emerging Interdisciplinary Areas, The Hong Kong University of Science and Technology, Hong Kong SAR, China, and also with the HKUST Shenzhen-Hong Kong Collaborative Innovation Research Institute, Futian, Shenzhen, China (e-mail: jun.ma@ust.hk).}
}

\usepackage{cite}

\usepackage{amsthm}
\usepackage{amsmath}
\usepackage{amsfonts}
\usepackage{mathtools}
\usepackage{amssymb}
\usepackage{graphicx}
\usepackage{epsfig,subfigure}
\usepackage{textcomp}
\graphicspath{{\images}}
\usepackage{siunitx}
\usepackage{ragged2e}
\usepackage{todonotes}
\usepackage{float}
\usepackage{lipsum}
\theoremstyle{plain}
\usepackage{xhfill}
\usepackage{booktabs}
\usepackage{soul}
\usepackage{hyperref}
\usepackage{courier}
\usepackage{algpseudocode}
\usepackage{algorithm}

\newcommand{\rr}{\mathop{{\rm I}\mskip-4.0mu{\rm R}}\nolimits}

\theoremstyle{definition}

\newtheorem{lemma}{Lemma}

\newtheorem{remark}{Remark}
 
\newtheorem{assumption}{\textbf{Assumption}}
\newtheorem{definition}{\textbf{Definition}}

\begin{document}

\maketitle
\thispagestyle{empty}
\pagestyle{empty}
%%%%%%%%%%%%%%%%%%%%%%%%%%%%%%%%%%%%%%%%%%%%%%%%%%%%%%%%%%%%%%%%%%%%%%%%%%%%%%%%
\begin{abstract}
Abrupt maneuvers by surrounding vehicles (SVs) can typically lead to safety concerns and affect the task efficiency of the ego vehicle (EV), especially with model uncertainties stemming from environmental disturbances. This paper presents a real-time fail-operational controller that ensures the asymptotic convergence of an uncertain EV to a safe state, while preserving task efficiency in dynamic environments. An incremental Bayesian learning approach is developed to facilitate online learning and inference of changing environmental disturbances. Leveraging disturbance quantification and constraint transformation, we develop a stochastic fail-operational barrier based on the control barrier function (CBF). With this development, the uncertain EV is able to converge asymptotically from an unsafe state to a defined safe state with probabilistic stability. Subsequently, the stochastic fail-operational barrier is integrated into an efficient fail-operational controller based on quadratic programming (QP). This controller is tailored for the EV operating under control constraints in the presence of environmental disturbances, with both safety and efficiency objectives taken into consideration. 
We validate the proposed framework in connected cruise control (CCC) tasks, where SVs perform aggressive driving maneuvers. The simulation results demonstrate that our method empowers the EV to swiftly return to a safe state while upholding task efficiency in real time, even under time-varying environmental disturbances.
\end{abstract}
%%%%%%%%%%%%%%%%%%%%%%%%%%%%%%%%%%%%%%%%%%%%%%%%%%%%%%%%%%%%%%%%%%%%%%%%%%%%%%%%
\section{Introduction}\label{sec:introduction} 
With the rapid advancement of autonomous driving technology, ensuring the safety and reliability of autonomous vehicles (AVs) has become a paramount concern~\cite{chen2022milestones,ma2022alternating}. One key underlying factor to this concern is the existence of model uncertainties resulting from unexpected environmental disturbances in high-speed driving scenarios~\cite{brandt2022high, knaup2023safe}, such as changing road grade and aerodynamic drag. These factors necessitate high-speed AVs to continually adapt to inevitable disturbances to achieve safe and efficient operations.
Additionally, the unpredictable maneuvers of surrounding vehicles (SVs), such as sudden deceleration, are challenging to anticipate. These unexpected maneuvers significantly influence the motion of the ego vehicle (EV), compromising driving efficiency and potentially leading to safety concerns~\cite{zheng2023real}. 
This requires the EV to not only navigate safely under normal conditions but also promptly and effectively respond to these disturbances, such that continuous operation can be ensured. To achieve this target, the EV must react swiftly to sudden maneuvers of SVs and adapt to time-varying environmental disturbances in real time.  
This necessitates the development of an efficient fail-operational controller that operates at over $50\,\text{Hz}$ to ensure that the EV reverts to a predefined safe state in the event of a fault (e.g., a safety violation) while maintaining its normal operational ability~\cite{stolte2021taxonomy}.

To effectively respond to hazardous situations, several works have focused on trajectory repairing for the EV~\cite{lin2021sampling,rosmann2012trajectory,ziegler2014trajectory}. 
While these works strive to replan infeasible trajectory segments to enhance safety, they lack formal safety assurance analysis. Furthermore, the repair frequency is slow, which is typically below $50\,\text{Hz}$. This limitation hinders the EV's capacity to react swiftly to abrupt maneuvers executed by SVs in high-speed scenarios. To address these issues, researchers have employed reachability analysis to ensure the safety of the EV~\cite{pek2020fail, wang2020infusing}. For instance, a low-level safety-preserving controller running at $50\,\text{Hz}$ has been developed to minimally intervene in unsafe actions based on the Hamilton-Jacobi reachability theory~\cite{wang2020infusing}. Additionally, control barrier functions (CBFs) have been adopted to ensure formal safety for safety-critical autonomous driving systems~\cite{ames2019control, he2022autonomous, zheng2024barrier}. These methods involve the computation of a forward invariance safe set, serving as a hard constraint to realize safe interactions between the EV system and dynamic SVs. Although these works can provide formal safety assurances for deterministic systems, they may encounter challenges in ensuring the safety of the EV in the presence of model uncertainties resulting from environmental disturbances.
  
To cater to environmental disturbances, robust CBF has been developed~\cite{robey2021learning,alan2022disturbance, nguyen2021robust}. In particular, to address the challenges posed by road and wind disturbances in high-speed autonomous driving scenarios, a disturbance observer-based safety-critical controller has been proposed for connected cruise control (CCC) tasks~\cite{alan2022disturbance}. However, determining an appropriate robust bound remains a challenge due to the need to strike a balance between robustness and feasibility\cite{nguyen2021robust}. 
On the other hand, researchers have explored Bayesian learning approaches to quantify environmental disturbances. These estimated disturbances have been leveraged to develop stochastic CBFs~\cite{zheng2020learning,fan2020bayesian, wu2022safe}, which provide formal safety analysis for uncertain systems. For instance, an adaptive CBF has been introduced to enable safety-critical high-speed Mars rover missions, incorporating tractable Bayesian model learning~\cite{fan2020bayesian}. Nonetheless, the learning process necessitates offline training due to its high computational complexity. To facilitate learning efficiency, an event-triggered mechanism is developed to update the Gaussian Process (GP) in model learning for safety-critical uncertain systems~\cite{wu2022safe}. Despite these advancements, none of these works addresses safety recovery for the EV under environmental disturbances. It is worthwhile to mention that the capability of safety recovery becomes pivotal in autonomous driving scenarios when sudden maneuvers by SVs propel the EV into an unsafe state in the presence of uncertain environmental disturbances. 

In this paper, we propose a real-time fail-operational controller for the EV in the presence of time-varying environmental disturbances. This controller is designed to guide autonomous vehicles back to a predefined safe state asymptotically, while upholding task efficiency.
First, we devise an incremental learning strategy to reduce the online learning complexity of GPs from $O(n^3)$ to $O(n^2)$, thereby enabling the EV to adapt online to changing environmental disturbances effectively. Subsequently, a stochastic fail-operational barrier is developed by utilizing CBF in conjunction with the estimated environmental disturbances obtained through the incremental learning process. Rigorous theoretical analysis of probabilistic asymptotic stability is provided with the aim of converging the unsafe EV back to a defined safe set.
Finally, we validate the effectiveness of the proposed fail-operational controller in a CCC task under time-varying environmental disturbances, demonstrating effective online learning and safety recovery for the EV.
 
The remainder of the paper is organized as follows: Section~\ref{sec:preliminaries_and_problem_formulation} presents the preliminaries and problem statement. In Section~\ref{sec:methodology}, we detail the proposed methodology. Section~\ref{sec:simulation} demonstrates the numerical simulation of the proposed algorithm on an uncertain CCC system. Finally, Section~\ref{sec:conclusion} summarizes the key findings and insights of this study.
\section{Preliminaries and problem statement}\label{sec:preliminaries_and_problem_formulation} 
In this study, we consider the class of uncertain discrete-time nonlinear systems for the EV described by 
\begin{equation}\label{eq:nonlinear_system}
    x_{k+1} = F(x_k,u_k) = f(x_k) + \mathcal{\psi}(x_k)u_k + w(x_k),
\end{equation}
where $x_k \in \mathcal{X}\subset \rr^n$,  $u_k \in \mathcal{U}\subset \rr^m$ and $w\in \mathcal{W} \subset \rr^n$ denote the state, control, and uncertain disturbance vectors, respectively; $k\in \mathbb{Z}_+ = \{0, 1, \cdots\}$. The system matrix $f : \mathcal{X}\ \rightarrow \rr^{n}$  and input matrix $ \mathcal{\psi}: \mathcal{U}\rightarrow \rr^{n\times m}$ are local Lipschitz continuous. 
We make the following assumptions to tackle the uncertain disturbances in~(\ref{eq:nonlinear_system}).
\begin{assumption}
\label{assumption:regularity} 
The uncertain disturbance vector $w$ has a bounded norm in the associated Reproducing Kernel Hilbert Space~\cite{scholkopf2002learning}, corresponding to a differentiable kernel $k$.
\end{assumption}
\begin{assumption}
\label{assumption:available_traj} 
The following collection of state-disturbance trajectories is available:  
\begin{equation}
  \mathcal{D}_N: =\left\{\left(x^{\left(i\right)},\ \Tilde{w}^{\left(i\right)}\right)\right\}_{i=1}^N,\ \Tilde{w}^{(i)}= w(x^{(i)}) +\upsilon_i,
\end{equation} 
where $N \in \mathbb{Z}_+ $ denotes the number of samples; $\Tilde{w}^{(i)} = [\Tilde{w}_1^{(i)}, \Tilde{w}_2^{(i)}, \cdots, \Tilde{w}_n^{(i)}]^T $ denotes the $i$-th measured disturbance vector ${w}(x^{(i)})=[{w}_1(x^{(i)}), {w}_2(x^{(i)}), \cdots, {w}_n(x^{(i)})]^T$ with independent and identically distributed white noise $\upsilon_i\sim \mathcal{N}\left(0,\sigma_{\text{noise}}^2I_n\right)$.
 
\end{assumption} 

\begin{remark}
Assumption~\ref{assumption:regularity} implies that the uncertain disturbance vector $w$ is regular to the kernel and has a certain level of smoothness. This assumption further indicates the matrix $F$ is local Lipschitz continuous on $\mathcal{X}$, ensuring the solution of (\ref{eq:nonlinear_system}) is unique and exits.
\end{remark} 

Consider the system (\ref{eq:nonlinear_system}), we  further define the unsafe set, safe set, safe boundary, and interior safe set by a $C^1$ function $h:  \rr^n\ \rightarrow \rr$ as follows:
\begin{subequations} \begin{align} Out(\mathcal{S}) &= \{x \in \mathcal{X} \mid h(x) < 0\}, \label{eq:unsafe} \\ \mathcal{S} &= \{x \in \mathcal{X} \mid h(x) \geq 0\}, \label{:safe} \\ \partial \mathcal{S} &= \{x \in \mathcal{X} \mid h(x) = 0\}, \label{eq:boundary_safe} \\ Int(\mathcal{S}) &= \{x \in \mathcal{X} \mid h(x) > 0\}. \label{eq:int_safe} \end{align} \end{subequations}

\begin{definition}
\label{def:extend_class_k}
 The continuous function $\gamma: (-c, d) \to (-c, d)$ is called an extended class $\mathcal{K}$ function for some $c, d \in \mathbb{R^{+}}$, if it is strictly increasing and satisfies the following conditions:
\begin{subequations} \begin{align}  \gamma(h(x))  & = \alpha  h(x) , \alpha \in (0, 1), \forall h(x) \neq 0,  \label{eq:constraints_class_k1}\\ \gamma(0) &= 0. \label{eq:constraints_class_k2}\end{align} \end{subequations}
\end{definition}

\begin{definition} (\cite{ahmadi2019safe,zeng2021enhancing})
\label{def:discrete_cbf}
 The  $C^1$ function $h$ is called a discrete-time control barrier function (CBF) for the set $\mathcal{S}$ defined in (\ref{eq:unsafe})-(\ref{eq:int_safe}), if there exists extend $\mathcal{K}$ functions $\gamma$ with $\mathcal{S} \subset  \mathcal{X}$, such that
 \begin{equation}\label{eq:discrete_cbf}
 \Delta  h(x_k) + \gamma(h(x_{k-1})) > 0,
\end{equation}
where $ \Delta  h(x_k)  :=   h(x_k) -   h(x_{k-1})$.
\end{definition} 

The goal of this work is to design a fail-operational controller for the uncertain nonlinear EV system (\ref{eq:nonlinear_system}) to accomplish specified tasks with the desired functionality, while satisfying the following two key objectives:
    \begin{enumerate}
    \label{goal}
      \item[1)]\textit{Online adaptivity}: The EV system (\ref{eq:nonlinear_system}) can continuously adapt to time-varying environmental disturbances using newly collected interaction data in real time.   
     \item[2)]\textit{Fail-operational control}: 
     The EV system (\ref{eq:nonlinear_system}) can asymptotically converge to the safe set $\mathcal{S}$ from an unsafe state $ Out(\mathcal{S})$, while upholding task efficiency.
    \end{enumerate}

\section{Methodology}\label{sec:methodology}
In this section, we first introduce an online incremental Bayesian learning approach to approximate the disturbances in the system  (\ref{eq:nonlinear_system}). Then, we develop a stochastic fail-operational control barrier based on the quantified disturbances learned from interaction data with mathematical proof. Finally,  We design an efficient fail-operational Quadratic Programming (QP) controller using stochastic optimization techniques.  
%%%%%%%%%%%%%%%%%%%%%%%%%%%%%%%%%%%%%%%%%%%%%%%%%%%%%%%%

%%%%%%%%%%%%%%%%%%%%%%%%%%%%%%%%%%

\subsection{Gaussian Process}
\label{section:GP}
As a typical Bayesian learning approach, the GP is a nonparametric method for learning complex functions and their uncertainty distributions~\cite{deisenroth2013gaussian}. In this study, we develop an incremental GP model that leverages Assumption~\ref{assumption:regularity} to learn the disturbance $w$ using collected interaction data from the environment during operation. 

Similar to~\cite{ostafew2016robust}, we assume disturbances are uncorrelated to train $n$ independent GPs to approximate the nonlinear function $w:X\rightarrow\mathbb{R}^n$ as follows:
\vspace{-1.5mm}
\begin{equation}
    \hat{{w}}\left(x\right)=\left\{\begin{matrix}{\hat{w}}_1\left(x\right)\sim\mathcal{N}(\mu_1\left(x\right),\ \sigma_1^{2}(x))\\
    {\hat{w}}_2\left(x\right)\sim\mathcal{N}(\mu_2\left(x\right),\ \sigma_2^{2}(x))\\
    \ldots\\{\hat{w}}_n\left(x\right)\sim\mathcal{N}(\mu_n\left(x\right),\ \sigma_n^{2}(x))\\\end{matrix}\right.. \vspace{-1.5mm}
\end{equation}

Given $N$ collected data pairs $ \mathcal{D}_{N}: =\left\{\left(x^{\left(i\right)},\ \Tilde{w}^{\left(i\right)}\right)\right\}_{i=1}^N$, the mean and variance of the $j$-th component $\hat{w}_j(x_*)$ at the query state $x_*$ can be inferred as: 
\begin{subequations}
\label{eq:gp_infer}
\begin{align} 
\mu_j(x_{*})&=k_{j}^{T}(K_{\sigma,j}+\sigma_{\text{noise}}^{2}I_N)^{-1}\Tilde{w}_{N,j},
\label{mean} \\
\sigma_j^{2}(x_*)&=k_j(x_*,x_{*})-k_{N,j}^{T}(K_{\sigma,j}+\sigma_{\text{noise}}^{2}I_N)^{-1}k_{N,j},
\label{var}
\end{align} 
\end{subequations}
where $\Tilde{w}_{N,j}=[\Tilde{w}_j^{(1)},\Tilde{w}_j^{(2)},\cdots,\Tilde{w}_j^{(N)}]^T\in\mathbb{R}^{N}$ denotes the observed vector. $K_{\sigma,j} \in\mathbb{R}^{N\times N}$ is the covariance matrix with entries $[K_{\sigma,j}]_{(i,q)}=k_j(x_{i},x_{q})$, $i,q \in\mathcal{I}_1^{N} = \left\{ 1, \cdots, N \right\}$, and $k_j(x_{i},x_{q})$ is the kernel function. $k_{N,j}=[k_j(x^{(1)},x_{*}),k_j(x^{(2)},x_{*}),\cdots,k_j(x^{(N)},x_{*})]^T\in\mathbb{R}^{N}$. 

\begin{lemma} 
\label{lemma:1}
(\cite{srinivas2012information, umlauft2018uncertainty}) 
Let $\varsigma\in(0,\ 1)$ and the measurement noise $\upsilon_j$ is uniformly bounded by $\sigma_{\text{noise}}$. Then a probability $Pr$ holds
	\begin{equation}
		\label{eq:prob}
		Pr\{\| \mu(x)-\delta(x) \| \leq  \|\beta\| \|\sigma(x)\|, \forall x \in \mathcal{X} \} \geq (1-\varsigma)^{2n} , 
	\end{equation}
	where $\beta = [\beta_1, \beta_2, \cdots, \beta_n]$, $\beta_j=(2\|{\delta_j\|^2}_{k_j}+300\gamma_j ln^3(\frac{N+1}{\upsilon_j}))^{-2}$; $\gamma_j$ is the maximum information gain  obtained about the GP prior from $N$ noisy samples as follows:
 \begin{equation}
     	 \gamma_j = \max\limits_{ x^{\left(1\right)},\ \ldots,\ x^{(N)} \in \mathcal{X}}\frac{1}{2}\log(\det(I_N - \frac{{K}_{\sigma,j}(x, x^{\prime})}{\sigma^{2}_{\text{noise}}})), 
 \end{equation}
 where $ x, x^{\prime} \in \{x^{\left(1\right)},\ \ldots,\ x^{(N)}\}$. 
\end{lemma}
  
\subsection{Incremental Learning for Enhanced GPs}
A significant challenge in the practical application of GPs is the computational burden associated with learning from large datasets. Incremental learning techniques can help to overcome this challenge through processing data streams in small increments,  instead of processing the entire dataset at once~\cite{wu2019large,zheng2022safe}. This technique enables the efficient updating of the GP model as new data becomes available, without retraining the entire model from scratch.  

\emph{1) Active Learning:}
To effectively acquire labeled data points that offer the most valuable information for the incremental learning process, an active learning strategy is utilized. This strategy facilitates the selective and strategic acquisition of labeled data points, optimizing the learning progress and concurrently reducing the computational load associated with processing extensive datasets.
In the context of incremental learning, we utilize uncertainty estimates provided by the GP model to select data points that are most informative for model updates. At each timestep,
% we identify the data point with the highest uncertainty based on the data distribution of the GP. Specifically, 
we calculate the uncertainty of each data point of the current kernel matrix using the covariance matrix provided by the GP model with the latest data point. Subsequently, we replace the least relevant data point, measured by the diagonal elements of the covariance matrix, with the latest data point for training purposes. This prioritization of labeling uncertain data points enables the incremental GP to focus on refining its predictions in regions of the input space where its confidence is low, thereby enhancing the estimation of disturbances in the current state.

We measure the relevance of data points using the squared Euclidean distance between each point and a new point of interest. To quantify this relevance, we employ a radial basis function (RBF) formulated as: 
\begin{equation}
    \label{eq:active_learning_kernel}
    k(x_i,x_\text{new}) =  \theta \exp \left(-\frac{1}{2l^2}\|x_i-x_\text{new}\|^2\right),
\end{equation}
where $x_\text{new}$ denotes the newly acquired data point, while $x_i$ corresponds to the $i$-th data point in our kernel matrix. The parameter $\theta$ signifies the signal variance, playing a crucial role in regulating the scale of the kernel's output. $l$ serves as the length scale parameter, dictating the rate where the similarity between data points diminishes with increasing distance.
 
\emph{2) Incremental Learning:} With the informative data point selected by the active learning strategy ,  we leverage the Woodbury matrix identity to efficiently update the GP's kernel matrix and its inverse in an incremental way.

We denote the current kernel matrix and its inverse as $K_{\sigma,\text{cur}} \in\mathbb{R}^{N\times N}$ and $K_{\sigma,\text{cur}}^{-1} \in\mathbb{R}^{N\times N}$, respectively. We assume the dataset has reached its predefined size. At each time step, we add a newly collected interaction data point to the dataset and simultaneously remove the one with the lowest similarity based on~(\ref{eq:active_learning_kernel}) from the current kernel matrix.

The current kernel matrix $K_{\sigma,\text{cur}}$ can be represented in block matrix form as:
\begin{equation}
K_{\sigma, \text{cur}}=\begin{bmatrix}k_0 &  k^T_{N-1} \\  k_{N-1}  & \Omega\end{bmatrix},
\end{equation}
where $k_0 \in \mathbb{R}$ and $ {k}_{N-1} \in \mathbb{R}^{N-1}$ represent the variance and covariance vector of the data point that exhibits the lowest similarity with the newly collected data point in the dataset, respectively; $\Omega \in \mathbb{R}^{(N-1) \times (N-1)}$ is the sub-matrix at the right bottom corner.
  
The inverse matrix of $K_{\sigma,\text{cur}}$ can be computed as:
\begin{equation}
K_{\sigma,\text{cur}}^{-1}=\begin{bmatrix}\rho_0 & \rho^T_{N-1} \\ \rho_{N-1} & S\end{bmatrix},
\end{equation}
where $\rho_{0} \in \mathbb{R}$, $\rho_{N-1} \in \mathbb{R}^{N-1}$, and  $S \in \mathbb{R}^{(N-1) \times (N-1)}$.

To derive the updated kernel matrix $K_{\sigma,\text{new}}$, the following procedure is implemented as outlined in~\cite{zheng2022safe}: Initially, we remove the least relevant data point chosen by the active learning strategy from the current dataset. Next, we compute the variance $k_{0,\text{new}} \in \mathbb{R} $ and the covariance vector $k_{N-1,\text{new}} \in \mathbb{R}^{N-1}$ corresponding to the newly introduced data point. The resulting new kernel matrix $K_{\sigma,\text{new}}$ is obtained as:  
\begin{equation}
K_{\text{new}}=\begin{bmatrix}\Omega &k_{N-1,\text{new}} \\ k^T_{N-1,\text{new}} & k_{0,\text{new}}+\sigma_{\text{noise}}^2I\end{bmatrix},
\end{equation}
where $\sigma_{\text{noise}}^2I$ is the noise covariance matrix.

The inverse matrix of $K_\text{new}$ can be computed using the Woodbury matrix identity as follows: 
\begin{equation}
\small
K_{\text{new}}^{-1}=\begin{bmatrix}P+ P k_{N-1,\text{new}} (P k_{N-1,\text{new}})^T Q & - P  k_{N-1,\text{new}} Q \\ 
-(P k_{N-1,\text{new}})^TQ  & Q\end{bmatrix},
\end{equation}
where $P=\Omega -\rho_{N-1}\cdot \rho^T_{N-1}\rho_{0}^{-1} \in \mathbb{R}^{(N-1) \times (N-1)}$ and $Q=(k_{0,\text{new}}+\sigma_{\text{noise}}^2I- k_{N-1,\text{new}}^T P  k_{N-1,\text{new}})^{-1} \in \mathbb{R}$.

Note that the Woodbury matrix identity involves several matrix multiplications, resulting in a computational complexity of $O(N^2)$ for this increment learning, where $N$ represents the size of the dataset. Therefore, this approach is computationally efficient compared to the traditional GP,
which requires the inversion of the entire kernel matrix and has a computational complexity of $O(N^3)$.
 
To further improve the learning performance, we optimize the kernel hyperparameters of the RBF kernel  $k_j$ (\ref{eq:gp_infer}) using the log-marginal likelihood function of the following form:
\begin{equation} 
\small
\begin{aligned} \log p(\Tilde{w}_{N,j}|X_N,  \Theta_j) &= \frac{1}{2}\Tilde{w}_{N,j}^T(K_N+\theta_{f,j}^2  I_N)^{-1}\Tilde{w}_{N,j} \\ & + \frac{1}{2}\log\left|K_N+\theta_{f,j}^2 I_N\right| + \frac{N}{2}\log(2\pi),
\end{aligned} 
\end{equation} 
where $\Tilde{w}_{N,j}=   [\Tilde{w}_j^{(1)},\Tilde{w}_j^{(2)},\cdots,\Tilde{w}_j^{(N)}]^T$ denotes the observed disturbances vector; $X_N=[x^{(1)},x^{(2)},\cdots,x^{(N)}]^T$ denotes the corresponding state vector; $\Theta_j = [\theta_{f,j},l_{f,j}]^T$ denotes the kernel hyperparameters of the following RBF function: 
\begin{equation}
    k_j(x^{(i)},\ x^{(j)}) = \theta_{f,j} \exp\left(-\frac{1}{l_{f,j}^{2}}||x^{(i)}-x^{(j)}||^2\right).
\end{equation} 

This optimization is performed online using a validation set as follows:  
\begin{alignat}{2}
    \small \vspace{0mm}
     \big(\theta_{f,j}^*,l_{f,j}^*\big) =   \displaystyle\operatorname*{minimize}_{\big(\theta_{f,j},l_{f,j}\big)\in\mathbb R\times \mathbb R}\quad
    & \log p(\Tilde{w}_{N,j}|X_{N}, \theta_{f,j}, l_{f,j}), \label{eq:bound_opt_1}\\   \vspace{0mm}
    \operatorname*{subject \ to}\quad\quad
    &\mathrlap{\theta_{f,j}(0) = \theta_{f,0},}\label{eq:bound_opt_2}\\
    &\mathrlap{l_{f,j}(0) =  l_{f,0},} \label{eq:bound_opt_3}\\ 
    &\mathrlap{\theta_{f, min} \leq \theta_{f,j} \leq \theta_{f, max}, }\label{eq:bound_opt_4}\\
    &\mathrlap{l_{f, min} \leq l_{f,j} \leq l_{f, max}, }\label{eq:bound_opt_5},
   \end{alignat} 
where $\theta_{f,0}$ and $l_{f,0}$ represent the initial kernel hyperparameters; $\theta_{f, min}$ and  $\theta_{f, max}$ represent the minimum and maximum value for the kernel hyperparameter $\theta_{f}$, respectively; $l_{f, min}$ and  $l_{f, max}$ represent the minimum and maximum value for the kernel hyperparameter $l$, respectively.

To solve this bound constrained optimization problem~(\ref{eq:bound_opt_1})-(\ref{eq:bound_opt_5}), we employ the L-BFGS-B optimization algorithm \cite{morales2011remark}. This iterative method approximates the inverse Hessian matrix of the objective function, allowing us to find the optimal set of hyperparameters while adhering to the specified bounds for each hyperparameter.
\begin{figure}[t]
    \centering
    \includegraphics[width=0.9\linewidth]{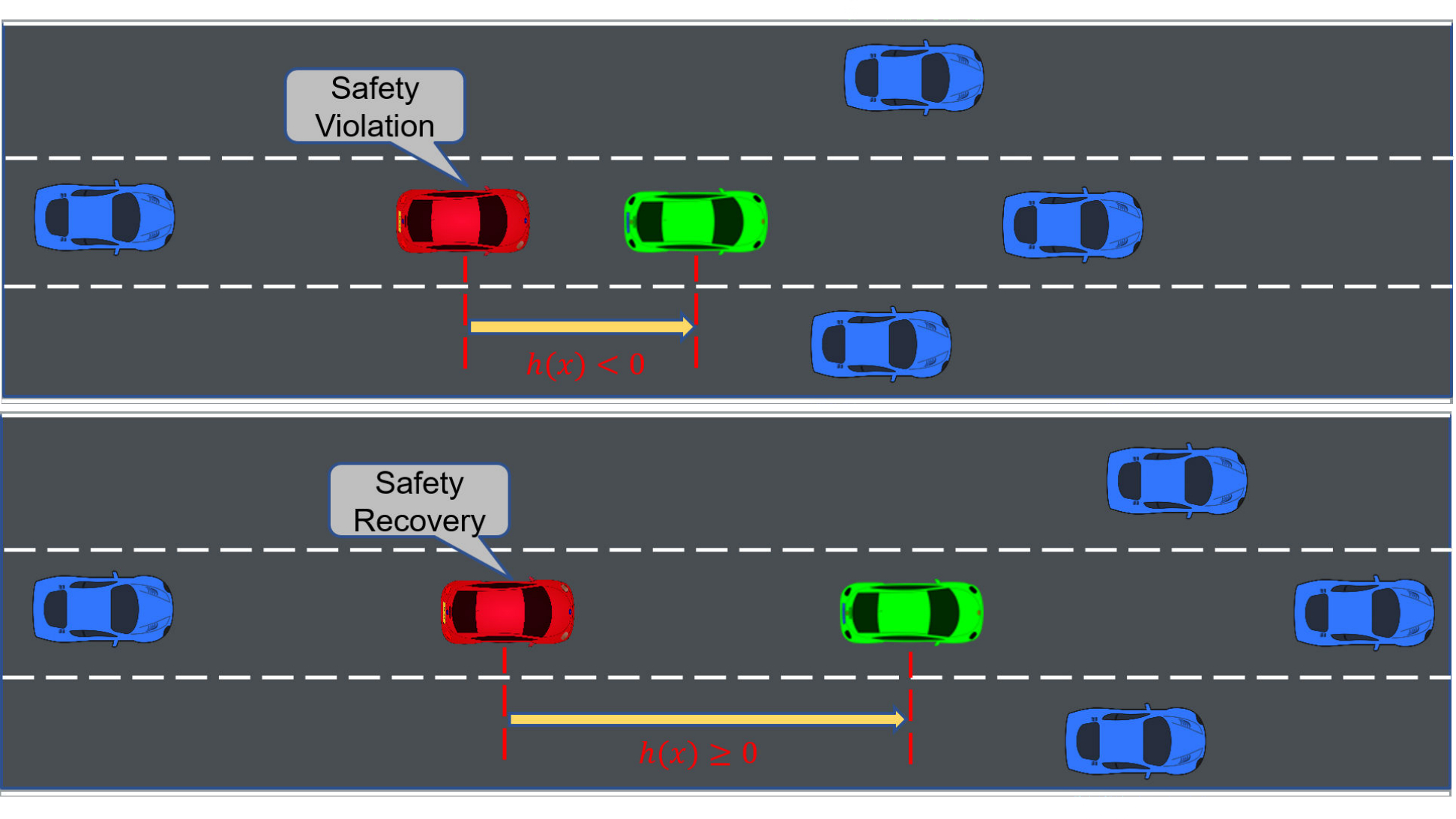}
    \caption{The stochastic fail-operational barrier module enables the red EV to recover from an unsafe state (top subfigure) to a safe state (bottom subfigure).}
\label{fig:safety_recovery}\vspace{-3mm}
\end{figure}
\subsection{Stochastic Fail-Operational Barrier}\label{sec:data_drive_representation_of_linear_systems} 
We aim to design a stochastic fail-operational barrier module that ensures the uncertain EV converges from the unsafe state $Out(\mathcal{S})$ to a safe state $\mathcal{S}$ and remains in the safe state after recovery, as depicted in Fig.~\ref{fig:safety_recovery}. 
\begin{lemma}
\label{lemma:2}
\it \textit{Let} $D_w = [\mu(x_{k-1}) - c \sigma(x_{k-1}), \mu(x_{k-1}) + c \sigma(x_{k-1})], c \in \mathbb{R}^+ $ represents the high-confidence disturbances set approximated by~(\ref{eq:gp_infer}) for the uncertain nonlinear system~(\ref{eq:nonlinear_system}) under  Assumptions 1-2. Then the unsafe state $x_0 \in Out(\mathcal{S})$ asymptotically converges 
to safe set $\mathcal{S}$ with probability at least $(1- \varsigma )^{2n}$ by the following constraint: 
\begin{align} \label{eq:stochastic_barrier_cons}
\begin{aligned}
 h(f(x_{k-1})) +& h(\psi(x_{k-1})  u_{k-1} + \epsilon(w(x_{k-1})) \\ &> h(x_{k-1})- \gamma(h(x_{k-1}),
\end{aligned}
\end{align} 
where $\epsilon(w(x_{k-1})) = h(\mu(x_{k-1})) - c \| h(\sigma(x_{k-1}))\|$, and the barrier function $h$ takes the form of an affine function. 
\end{lemma}
\begin{proof}
From Lemma~\ref{lemma:1}, we obtain
\begin{equation}\label{eq:stochatic_barrier_set}
       Pr\{w(x_{k-1}) \in D_w \} \geq (1- \varsigma )^{2n}.
\end{equation} 
Utilizing the properties of the affine barrier function $h$, we obtain:
\begin{align}
\label{eq:stochastic_barrier_set}
\begin{aligned}
   Pr\{  h(w(x_{k-1})) \geq \epsilon(w(x_{k-1})) \} \geq (1- \varsigma )^{2n}.
\end{aligned}
\end{align}   
Consequently, the following result holds with a probability of at least $(1 - \varsigma)^{2n}$: 
\begin{align}
\begin{aligned}
 & h(f(x_{k-1})) + h(\psi(x_{k-1})  u_{k-1} )+  h(w(x_{k-1}))\\ 
  &\geq   h(f(x_{k-1})) + h(\psi(x_{k-1}))  u_{k-1} + \epsilon(w(x_{k-1})) \\
  &    > h(x_{k-1})- \gamma(h(x_{k-1})).  
\end{aligned}
\end{align}   
This yields:
\begin{equation}\label{eq:stochatic_constraint}
       Pr\{  h(x_k)  > h(x_{k-1})- \gamma(h(x_{k-1})\} \geq (1- \varsigma )^{2n}.
\end{equation} 
With Definition~\ref{def:extend_class_k}, we deduce the following results with probability at least $(1- \varsigma )^{2n}$ : 
\begin{equation}\label{eq:stochatic_barrier_set1}
\begin{array}{rcl}
\displaystyle
       &h(x_k)- ( h(x_{k-1})  - \gamma(h(x_{k-1}))  \\
         & =  h(x_k) - (1-\alpha)h(x_{k-1}) > 0 .  
\end{array}
\end{equation}
Hence, $ h(x_k) > (1-\alpha)h(x_{k-1})$, where $\alpha \in (0,1)$. It yields:
\begin{equation}\label{eq:stochatic_barrier_set2}
       Pr\{ h(x_k) \geq  (1-\alpha)h(x_{k-1}) \} \geq (1- \varsigma )^{2n}.
\end{equation}
This result indicates that the state of the EV with $h(x_0) < 0$ will asymptotically converge to the safe set $\mathcal{S}$ at a rate of at least $(1-\alpha)^k$ over $k$ steps of evolution.
\end{proof}

\begin{remark}
\label{remak:forward_invariance}
The constraint presented in (\ref{eq:stochatic_barrier_set2}) has significant implications for the behavior of the uncertain nonlinear system (\ref{eq:nonlinear_system}) when the system is inside the safe set $\mathcal{S}$, where $h(x_k) \geq 0$. This constraint ensures that, once the system enters the safe set, it remains within this region, guaranteeing the forward invariance of the safe set $\mathcal{S}$. This critical property has been discussed in detail in~\cite{zeng2021enhancing}.
\end{remark}
%%%%%%%%%%%%%%%%%%%%%%%%%%%%%%%%%%
  
%%%%%%%%%%%%%%%%%%%%%%%%%%%%%%%%%%
\subsection{Real-Time Fail- Operational Controller}\label{sec:data-driven-st-mpc}
The fail-operational controller aims to achieve the desired task performance specified by fail-operational control criteria, as outlined in~\cite{stolte2021taxonomy}.
This requires the fail-operational controller to repair the undesired state while upholding task efficiency. Considering the input constraints of the uncertain nonlinear EV system~(\ref{eq:nonlinear_system}), we introduce the following computational-efficiency fail-operational controller in a QP formulation:
% \begin{alignat}{2}
% \small
% \hspace{-5mm} & u_{k}^* = \displaystyle\operatorname*{argmin}_{ u_k \in\mathbb{R}^m}\quad
%     && \|u_k\|^2 + \lambda_{\zeta} \zeta^2+ \lambda_{\iota} \iota^2, \label{eq:opt_1}\\
% & \operatorname*{subject \ to}\quad\quad
%     && \mathrlap{h(f(x_{k})) +  h(\psi(x_{k}) ) u_{k} + \epsilon(w(x_{k})) } \nonumber \\ 
%     &&& > h(x_{k})- \gamma(h(x_{k})) -\zeta, \label{eq:opt_2}\\  
%     &&&  \Delta V(x_k)  + c_v V(x_{k})  < \iota, \label{eq:opt_3}\\ 
%     % &&& \zeta \geq 0, \label{eq:opt_4},\\
%     &&& u_\text{min} \leq u_{k} \leq u_\text{max}\label{eq:opt_5},
% \end{alignat} 
\begin{align}
\hspace{-5mm} u_{k}^* =& \operatorname*{argmin}_{ u_k \in\mathbb{R}^m} && \|u_k\|^2 + \lambda_{\zeta} \zeta^2+ \lambda_{\iota} \iota^2, \label{eq:opt_1}\\
& \text{subject to}  && h(f(x_{k})) +  h(\psi(x_{k}))  u_{k} + \epsilon(w(x_{k})) \nonumber\\
&&& > h(x_{k})- \gamma(h(x_{k})) -\zeta, \label{eq:opt_2}\\  
&&&  \Delta V(x_k)  + c_v V(x_{k}) < \iota, \label{eq:opt_3}\\ 
&&& u_\text{min} \leq u_{k} \leq u_\text{max}\label{eq:opt_5},
\end{align} 
where $u_\text{min}$ and  $u_\text{max}$ denote the minimum and maximum control input value, respectively; $\zeta, \iota\in \rr^+$ are non-negative slack variables used to ensure the feasibility of the constrained optimization problem (\ref{eq:opt_1})-(\ref{eq:opt_5}); $\lambda_{\zeta} $ and $ \lambda_{\iota}  \in \mathbb{R}^+$ are corresponding weights; $\Delta V(x_k) = V(x_{k+1}) - V(x_k)$, where $V$ is a discrete-time exponentially stabilizing control Lyapunov function (ES-CLF)~\cite{agrawal2017discrete} utilized to encode the desired state for the EV, and the constraint (\ref{eq:opt_3}) is specifically crafted to stabilize the uncertain nonlinear EV system (\ref{eq:nonlinear_system}) toward this desired state, which can be further transformed into a deterministic constraint with the estimated disturbances set $D_w$ based on~\cite{castaneda2021gaussian}.
% \begin{align}
% \label{eq:esclf}
% \begin{aligned} 
%  &  V(f(x_{k})) + V(\psi(x_{k})) u_{k} + V(\mu(x_{k})) \\ 
%    &  - c \| V(\sigma(x_{k}))\| + c_v V(x_k) < \iota,  
% \end{aligned}
% \end{align}  
% where  $c_v\in (0,1)$ is the coefficient to adjust the stabilizing speed.

% \begin{proof}
% According to the dentition of the discrete-time ECLF $V$ (\textbf{Definition 1} of \cite{agrawal2017discrete}), we obtain:
% \begin{align}
% \label{eq:stochastic_barrier_set}
% \begin{aligned}
%  c_1 ||x_{k+1}||^2 \leq V(x_{k+1}) \leq  c_2 ||x_{k+1}||^2,
% \end{aligned}
% \end{align}   
% where $c_1, c_2 \in \rr^+$. With the \textbf{Lemma 1}, we can get
% \begin{equation}\label{eq:stochatic_barrier_set}
%        Pr\{ ||x_{k+1}|| \in [ {x}_{low, k+1} ,   {x}_{up, k+1} ]\} \geq (1- \varsigma )^{2n},
% \end{equation} 
% where $  {x}_{low, k+1} = f(x_k) +\psi(x_k)u_k +\mu(x_k)- c\sigma(x_k)$, $ {x}_{up, k+1}  = f(x_k)+\psi(x_k) u_k +\mu(x_k)+ c\sigma(x_k)$

% Consequently, with a probability of at least $(1 - \varsigma)^{2n}$, the following result holds: 
% \begin{align}
% \begin{aligned}
%  V(x_{k+1}) \leq  & V(f(x_{k}))  + V(\mu(x_{k}))  + c^2\|V(\sigma(x_k)) \|   \\ 
%   &+ V(\psi(x_{k})u_{k}).  
% \end{aligned}
% \end{align}   
% This yields:
% \begin{equation}\label{eq:stochatic_constraint}
%        Pr\{  \sup_{w \in D_w}[ \Delta V(x_k)  + c_v V(x_{k})] < 0\} \geq (1- \varsigma )^{2n}.
% \end{equation}   
% This conclude the constraint~\ref{eq:esclf}.
% \end{proof}

\begin{remark}\label{remark:computing_terminal_region}
The parameters $\lambda_{\zeta}$ is assigned a large penalization weight to enforce the constraint $\zeta$ to be a negligible value, thereby minimizing its influence on the stochastic fail-operational barrier constraint~(\ref{eq:stochastic_barrier_cons}). Furthermore, $\lambda_{\zeta}$ is greater than $\lambda_{\iota}$ to prioritize safety over task performance.
According to Lemma~\ref{lemma:2}, when constraint~(\ref{eq:stochastic_barrier_cons}) is satisfied, the fail-operational control obtained by solving the QP problem~(\ref{eq:opt_1})-(\ref{eq:opt_5}) can effectively navigate the EV from an unsafe state $x \in Out(\mathcal{S})$ back to the safe set $\mathcal{S}$ 
with a high probability of at least $(1-\varsigma)^{2n}$.
\end{remark} 
  
\section{Illustrative Example}\label{sec:simulation}
In this section, we evaluate the effectiveness of the proposed real-time fail-operational controller in CCC driving tasks. The CCC system consists of one EV and four human-driven vehicles (HVs) exhibiting sudden acceleration and deceleration behaviors in the presence of time-varying environmental disturbances.

\subsection{Vehicle Model}
The uncertain nonlinear EV system dynamics are formulated as follows: 
\begin{equation}
\label{eq:ev_model}
\Dot{x}  = \left[\begin{array}{c}
\dot{p} \\
\dot{v} 
\end{array}\right]=\left[\begin{array}{c}
v_E\\
-\frac{F_f + F_r}{M} - a(\phi)
\end{array}\right] + \left[\begin{array}{c}
0\\
\frac{1}{M}
\end{array}\right]
u_E,
\end{equation}
where $M$ denotes the mass of the EV; $u_E \in [-0.3gM, 0.3gM]$ denotes the control input, $g$ is the gravitational acceleration; $p$ and $v$ represent the position and velocity of the EV, respectively. $F_f$, $F_r$, and $a(\phi)$ correspond to the aerodynamic drag, rolling resistance, and road grade, defined as follows: 
\begin{equation}
F_f = k_v v^2, \quad F_r = k_f (t) g M\cos(\phi), \quad a(\phi) = g \sin(\phi),
\label{eq:disturbance}
\end{equation} 
where $k_v$ and $k_f$ represent the coefficients for aerodynamic drag and road resistance, respectively; $\phi$ represents the road grade. The $k_f$ is assumed to be a constant value of 0.06, while $k_v$ and $\phi$ are set to zero to introduce uncertain disturbances for the EV. We discrete the systems (\ref{eq:ev_model}) using the Euler method with a discrete interval of $T_s = 0.02 \,\text{s}$. 
 
We define the state of the $i$-th HV  as $O^i = [s^i, v^i]^T$, where $ s^i$ and $v^i$ represent the position and velocity of the $i$-th HV, respectively. The desired control input of the $i$-th HV in the CCC system are adopted from~\cite{he2018data}: 
\begin{equation}
    u^{i}_k= \alpha_{i}(K_i(d^i_k) - v^i_k )+ \beta_i(v^{i+1}_k - v^i_k),
\end{equation}
where $d^i_k = s^{i+1}_k - s^i_k - l_i$ denotes the headway of the $i$-th HV at time step $k$; $l_i$ denotes the length of the $i$-th HV; $\alpha_{i}$ and $\beta_i$ denote the control gain coefficients of the $i$-th HV. The range function $K_i$ is used to describe the target velocity for the  $i$-th HV as follows:
\begin{equation}
    K_i(d^i_k) =
\begin{cases}
0 & \text{if } d^i_k \leq d^i_{\text{min}}, \\
k_i( d^i_k -d^i_{\text{min}} ) & \text{if } d^i_{\text{min}} < d^i_k < d^i_{\text{max}}, \\
v_{\text{max}} & \text{if } d^i_k \geq d^i_{\text{max}},
\end{cases}
\label{eq:hv_model}
\end{equation} 
where $v_\text{max}$ denotes the maximum velocity and  $k_i = \frac{v_\text{max}}{d^i_{\text{max}} -d^i_{\text{min}}}$. The small headway $d^i_{\text{min}}$ and large headway $d^i_{\text{max}}$ indicate where the $i$-th HV intends to stop and travel, respectively. The vehicle parameters are listed in Table~\ref{table:Parameter_Settings}.

\begin{table}[tp]
\centering
\scriptsize
\caption{Parameters of Vehicle Model}   \vspace{0mm}
\label{table:Parameter_Settings}
\begin{tabular}{c c c c}
\hline
\hline
$k_v$ & $0.25\, \text{N} \cdot \text{s}^2/\text{m}^2$ & $\alpha_i$   & $ 30 $\\
$\beta_i$ & $2000$ &$l_i$  & $2.91\,\text{m}$\\
$M$  & $1650\,\text{kg}$  & $g$    &  $9.81\,\text{m/s}^2$\\
 $ \phi_\text{min}$  & -10\,\text{deg}\ & $ \phi_\text{max}$  & 10\,\text{deg}\\
$v_\text{max}$  & $40\,\text{m/s} $  \\
\hline
\hline
\end{tabular} 
\end{table}
     \begin{table}[tp]
\caption{Average computation time for different initial states in the CCC task under time-varying environmental disturbances.}  \vspace{-2mm}
\label{tab:avg_time}
\begin{center}
    \begin{small} 
        \begin{tabular}{ c||c|c|c}
            \toprule 
            Initial State $x_0$ & QP solving  & Learning & Inference  \\
            % \midrule
            \hline
            ${[ 25 \,\text{m}, 18\,\text{m/s}]}^T$ & 2.349 $\,\text{ms} $ & 4.696 $\,\text{ms} $& 0.032 $\,\text{ms} $\\
            % \midrule
            \hline
           ${[ 110 \,\text{m}, 18\,\text{m/s}]}^T$  & 3.035 $\,\text{ms} $ &  4.241 $\,\text{ms} $ & 0.034 $\,\text{ms} $\\ 
            \bottomrule
        \end{tabular}
    \end{small}
\end{center}\vspace{-2mm}
\end{table} 
 \begin{figure}[tp]
	 	\centering  
                \subfigure[]{
			     \label{fig:QP_solving_1}
\includegraphics[width=0.9\linewidth]{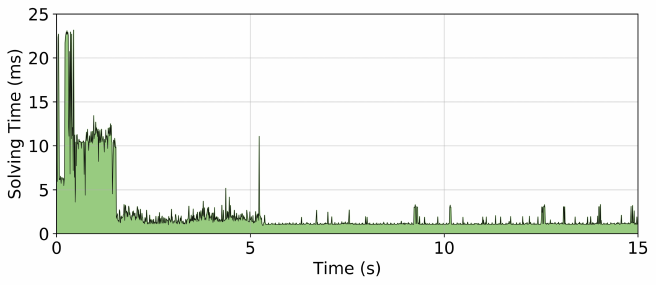}}\hspace{-0mm} 	\vspace{-3mm}
               \subfigure[]{
			\label{fig:QP_solving_2}
\includegraphics[width=0.9\linewidth]{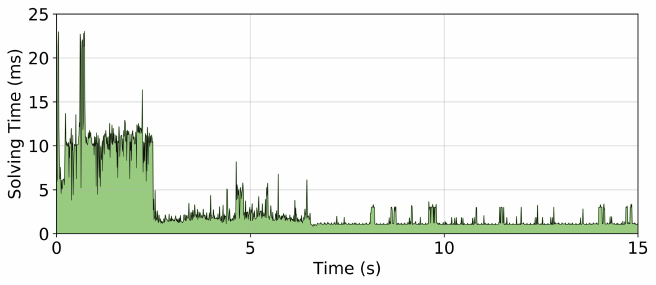}}\hspace{-0mm} \vspace{-0mm}
	 	\caption{ The evolution of solving time of the optimization problem (\ref{eq:opt_1})-(\ref{eq:opt_5}) with two unsafe initial states. (a) Initial state $x_0 = {[ 25 \,\text{m}, 18\,\text{m/s}]}^T $, (b) Initial state $x_0 = {[ 110 \,\text{m}, 18\,\text{m/s}]}^T $. }		\vspace{-2mm}
	 	\label{fig:QP_time_performance}
	\end{figure}
\subsection{Simulation Setup}
Our simulation experiments were conducted on an Ubuntu 20.04 LTS system with an AMD Ryzen 7 5800H CPU with eight cores and sixteen threads. It operates at a base clock speed of 2.28 GHz, with a maximum boost frequency of 3.20 GHz and a minimum frequency of 1.20 GHz.  The system is equipped with 16 GB of RAM.  
We utilize the CVXOPT as the solver for the QP problem (\ref{eq:opt_1})-(\ref{eq:opt_5}) based on Python 3.7.
 
The initial states of the HVs are set as $O^1_0={[ 240 \,\text{m}, 18 \,\text{m/s} ]}^T $, $O^2_0={[ 180 \,\text{m}, 18 \,\text{m/s}]}^T$, $O^3_0={[120\,\text{m}, 18 \,\text{m/s}]}^T$, $O^4_0={[ 0\,\text{m}, 18\,\text{m/s}]}^T$.  
The EV between the third and fourth HV aims to cruise at a target speed $v_g = 20 \,\text{m/s}$ in a one-direction road while keeping a desired following distance $[d_1,d_2]$ with its front HV. To achieve this goal, we design four independent GPs to model the state disturbances for the uncertain EV and its front HV. The following CBF and ES-CLF functions are designed:
\begin{subequations}
\label{eq:exp_cbf}
\begin{align} 
h_1 (x_{k})& = s^3_k -  p_k - d_1 ,
\label{exp:cbf1} \\
h_2 (x_{k}) &= -s^3_k +  p_k + d_2 ,
\label{exp:cbf2} 
\end{align} 
\end{subequations}
\begin{equation}
V(x_k) = \| v- v_d\|^2.
\end{equation}   
The following parameters are used : $N= 20$, $\sigma_{\text{noise}} = 10^{-6}$, $ \theta_{f,0} = 1$,  $\theta_{f, min} = 10^{-3}$, $\theta_{f, max} = 10^3$, $l_{f, 0} = 1$,  $l_{f, min} = 10^{-2}$, $l_{f, max} = 10^2$, $c =3$, $\alpha = 0.05$, $c_v = 0.8$, $\lambda_{\zeta} = 10^{30}$, $\lambda_{\iota} = 10^{10}$,
 $d_1=d^i_{\text{min}}=25 \,\text{m}$, $d_2=d^i_{\text{max}}=100 \,\text{m}$ and $v_d =20\,\text{m/s}$. The simulation duration and control frequency are set at $15\,\text{s}$ and $50\,\text{Hz}$, respectively. 
    	\begin{figure}[tp]
	 	\centering   
                \subfigure[Learning Time]{
			     \label{fig:IGP_learning1}
	\includegraphics[width=0.9\linewidth]{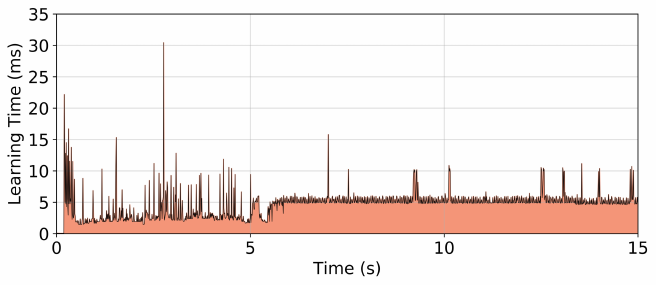}}\hspace{0mm}	\vspace{-3mm}
               \subfigure[Inference Time]{
			\label{fig:IGP_infer1}
		\includegraphics[width=0.9\linewidth]{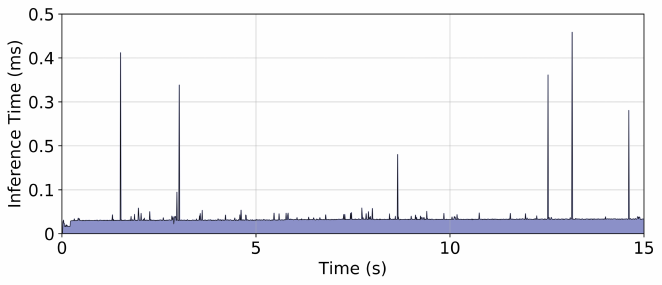}}\hspace{0mm}  
	 	\caption{The evolution of incremental learning and inference time with the initial state $x_0 = {[ 110 \,\text{m}, 18\,\text{m/s}]}^T $ . }		\vspace{-2mm}
	 	\label{fig:GP_time_performance1}
	\end{figure}
    	\begin{figure}[tp]
	 	\centering   
                \subfigure[Learning Time]{
			     \label{fig:IGP_learning2}
		\includegraphics[width=0.9\linewidth]{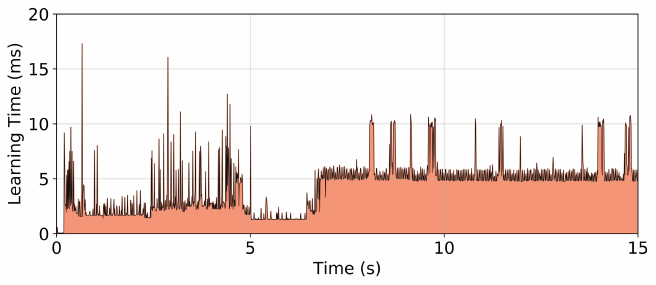}}\hspace{0mm}	\vspace{-3mm}
               \subfigure[Inference Time]{
			\label{fig:IGP_infer2}
		\includegraphics[width=0.9\linewidth]{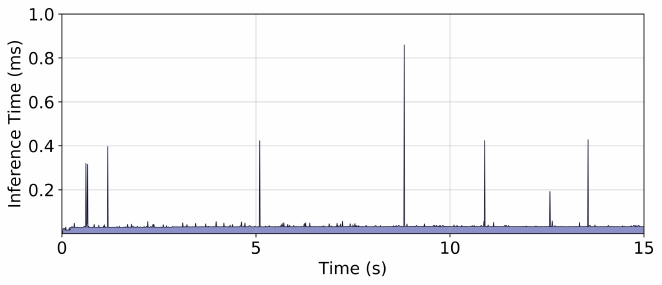}}\hspace{0mm} 
	 	\caption{ The evolution of incremental learning and inference time with the initial state $x_0 = {[ 25 \,\text{m}, 18\,\text{m/s}]}^T$. }		\vspace{-2mm}
	 	\label{fig:GP_time_performance2}
	\end{figure}
\subsection{Results}   
The initial state $x_0$ of the EV is set as ${[ 25 ,\text{m}, 18,\text{m/s}]}^T $ and ${[110 ,\text{m}, 18,\text{m/s}]}^T$, leading to two unsafe initial state configurations with $h_2 <0$ and $h_1<0$ for the EV, respectively.
We assess the real-time and task performance in achieving a safe following distance and desired cruise speeds in the presence of time-varying environmental disturbances.
\subsubsection{Real-Time Performance} 
    	\begin{figure}[tp]
	 	\centering   
                \subfigure[]{
			     \label{fig:cbf1}
			\includegraphics[width=0.9\linewidth]{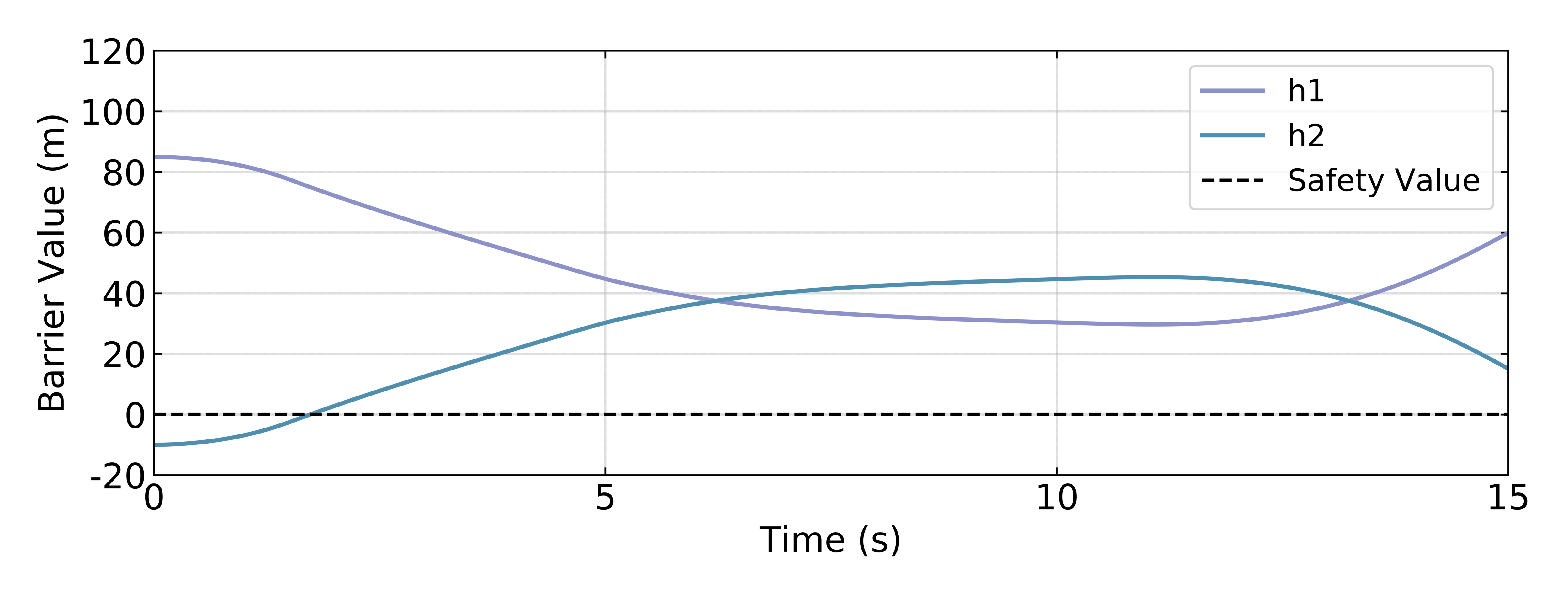}}\hspace{-0mm}	\vspace{-1mm}
               \subfigure[]{
			\label{fig:cbf2}
			\includegraphics[width=0.9\linewidth]{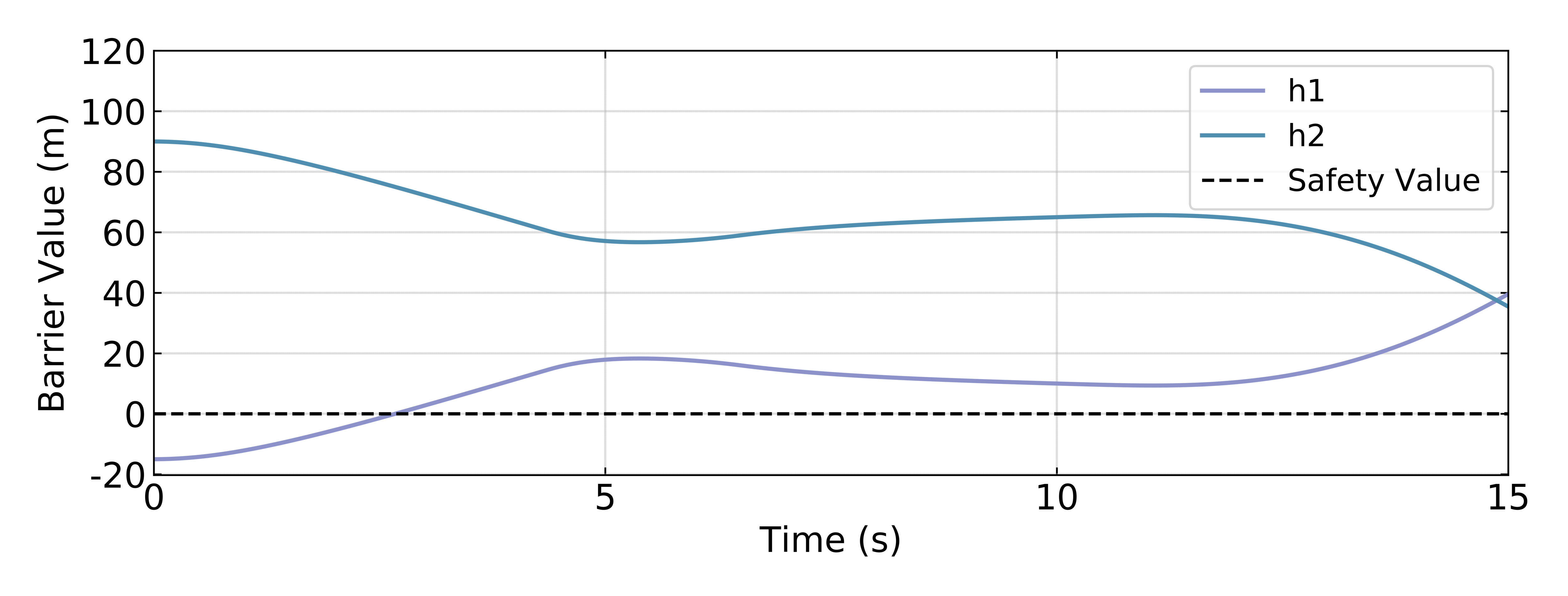}}\hspace{-0mm}  \vspace{-1mm}
	 	\caption{The evolution of the CBF value for the EV with two different unsafe initial states. (a) $x_0 = {[ 25 \,\text{m}, 18\,\text{m/s}]}^T $, (b) $x_0 = {[ 110 \,\text{m}, 18\,\text{m/s}]}^T $. The negative CBF can quickly converge to positive values and remain positive throughout the CCC task in the presence of environmental disturbances.}		\vspace{-3mm}
	 	\label{fig:CCC_barrier_value}
	\end{figure}
  \begin{figure}[tp]
    \centering
    \includegraphics[width=0.9\linewidth]{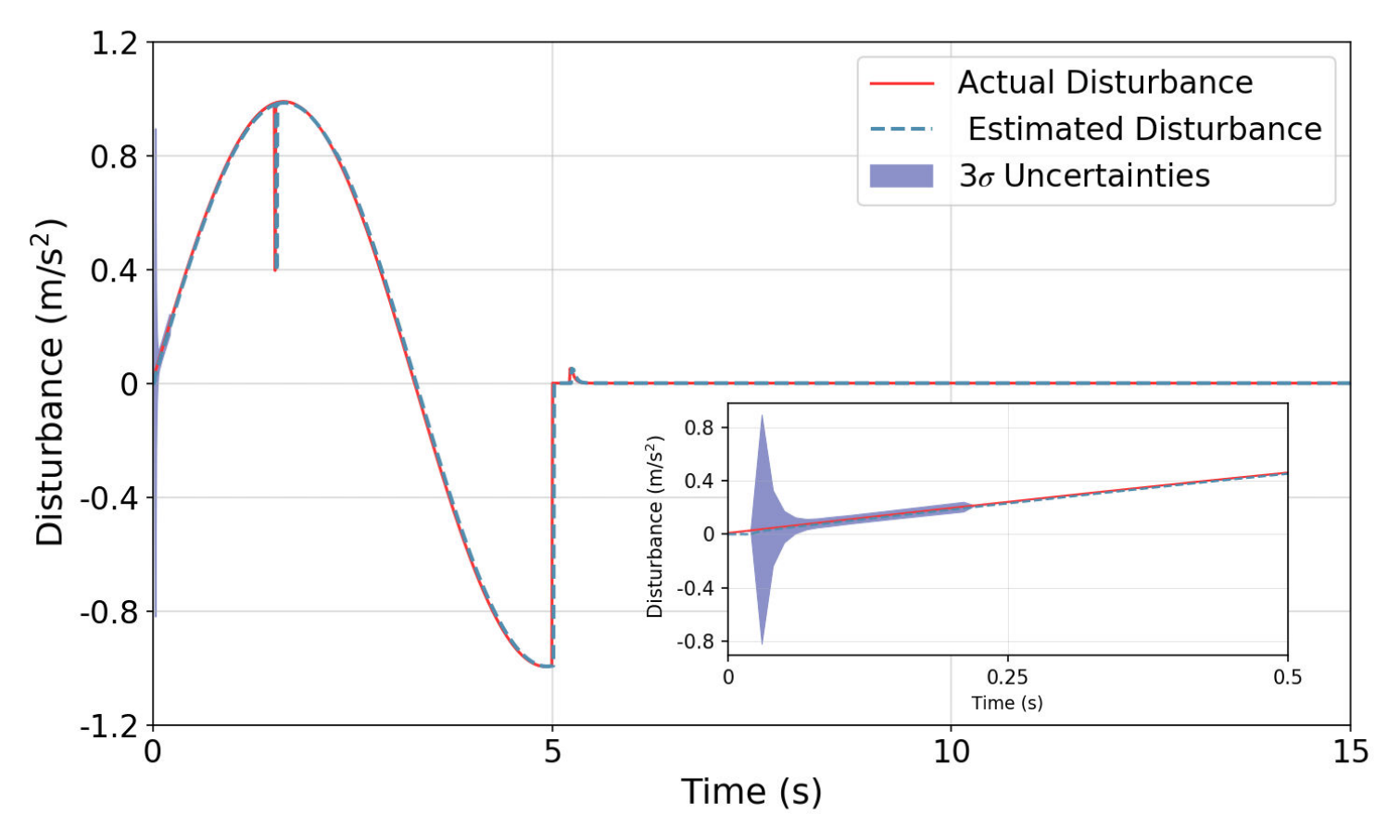} \vspace{-2mm}
    \caption{The estimated time-varying disturbance of the EV in the acceleration aspect, with an initial state of $x_0 = \begin{bmatrix} 25\,\text{m} \ 18\, \text{m/s} \end{bmatrix}$. The abrupt disturbance fluctuations result from a sudden change in the road resistance coefficient. The embedding figure illustrates the evolution of model disturbance from $0\,\text{s}$ to $0.5\,\text{s}$.}  
    \label{fig:CCC_uncertainties} \vspace{-3mm}
\end{figure}
Table~\ref{tab:avg_time} and Fig.~\ref{fig:QP_time_performance} depict the average and evolution of the computation time for the fail-operational controller (\ref{eq:opt_1})-(\ref{eq:opt_5}), with different initial states. The average solving times are $2.349\,\text{ms}$ and $3.035\,\text{ms}$ for initial state $x_0 = {[ 25 \,\text{m}, 18\,\text{m/s}]}^T$ and $x_0 = {[ 110 \,\text{m}, 18\,\text{m/s}]}^T$, respectively.  
Regarding the incremental learning process and inference shown in Fig. \ref{fig:GP_time_performance1} and Fig. \ref{fig:GP_time_performance2}. It is evident that both learning and inference time remain consistently low, with averages of less than $5\,\text{ms}$ and $0.4\,\text{ms}$, respectively.  
Notably, these durations collectively sum to less than $20 \, \text{ms}$ on average, thus ensuring the feasibility of real-time optimization, learning, and inference. Moreover, one can notice that the optimization time is relatively large during the intervals from $0\,\text{s} $ to $1.74\,\text{s}$ and  $2.54\,\text{s}$,  as shown in Fig.~\ref{fig:QP_solving_1} and Fig.~\ref{fig:QP_solving_2}, respectively. 
During these specific time intervals, the fail-operational controller diligently endeavors to restore the EV to a safe state within its designated safe set, characterized by the conditions $h_1 > 0$ and $h_2 > 0$, as depicted in Fig.~\ref{fig:CCC_barrier_value}. 

\subsubsection{Task Performance} 
Consider the EV's initial state denoted as  $x_0 = {[ 25 \,\text{m}, 18\,\text{m/s}]}^T $ as a case in point.  
The incremental learning performance in modeling the disturbances is illustrated in Fig.~\ref{fig:CCC_uncertainties}. One can notice that the estimated high-confidence uncertainties (at the 3$\sigma$ level) exhibit an initial decrease, maintaining a consistently low value. This indicates that the proposed incremental learning method is capable of effectively adapting to uncertain disturbances, leveraging real-time interaction data.
\begin{figure}[tp]
    \centering
    \includegraphics[width=0.9\linewidth]{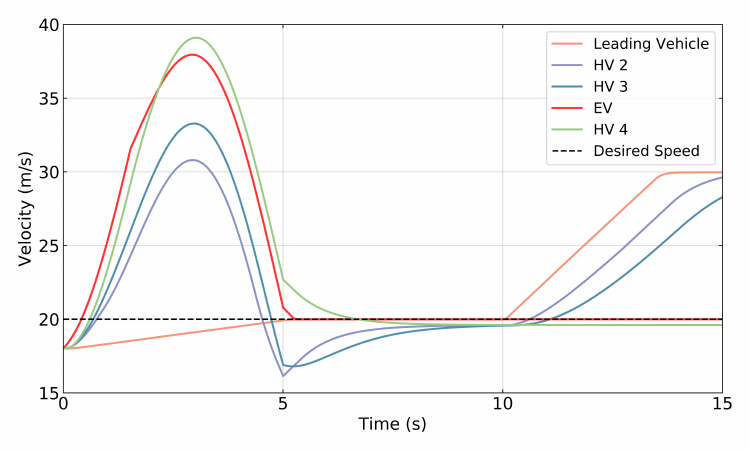}  \vspace{-2mm}
    \caption{The evolution of velocity for each vehicle in the CCC system led by the first HV. The similarities in the velocity profiles indicate how other vehicles attempt to adjust their speed to follow their front vehicle.
    }  
    \label{fig:CCC_velocity} \vspace{-3mm}
\end{figure}  
\begin{figure}[tp]
    \centering
    \includegraphics[width=0.9\linewidth]{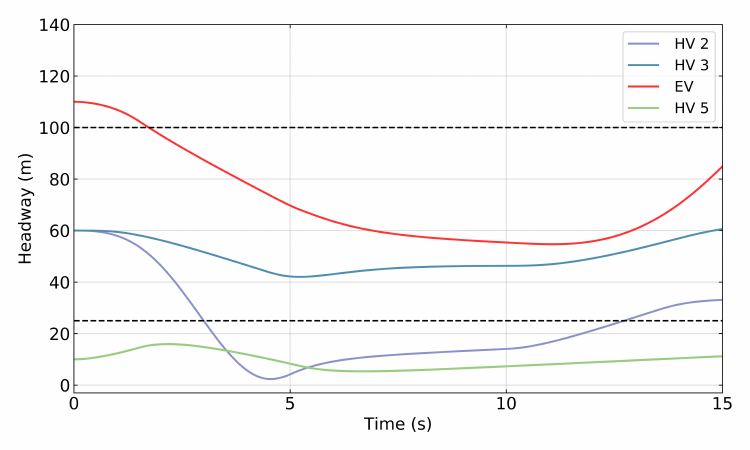} \vspace{-2mm}
    \caption{The headway evolution for other vehicles in the CCC system led by the first HV, where the two dashed black lines denote the target headway.}  
    \label{fig:CCC_headway} \vspace{-3mm}
\end{figure} 
 
Figure~\ref{fig:CCC_velocity} illustrates the temporal evolution of driving speeds for each vehicle. The EV accelerates to follow its front vehicle (HV 3) in the very beginning to drive its unsafe state back to the safe state, as evidenced by the CBF value $h_2$ and the EV's headway, which are depicted in Fig. \ref{fig:cbf1} and Fig. \ref{fig:CCC_headway}, respectively.   
In the presence of road and aerodynamic drag disturbances, as shown in Fig. \ref{fig:CCC_uncertainties}, the distance between HV 2 and its front leading vehicle (HV 1) falls below the desired headway threshold of $25\,\text{m}$ at $3.5\,\text{s}$. Consequently, the HV 2 promptly reduces its speed to maintain a safe following distance, causing an urgent deceleration of HV 3 from $3.5\,\text{s}$ to $5\,\text{s}$. As expected, the EV also reduces its speed to ensure a safe following distance. Notably, once it returns to a safe state, the EV consistently maintains the desired headway distance, whereas HV 2 and HV 5 struggle to maintain the desired following distance, as shown in Fig. \ref{fig:CCC_headway}. 
These observations underscore the capability of the EV to return to a safe state even in the presence of road and air drag disturbances and rapid acceleration and deceleration behaviors exhibited by HVs. 

In terms of task accuracy, the EV can quickly achieve a desired cruise speed $v_d = 20\,\text{m/s}$ around $5\,\text{s}$, while keeping a desired following distance with its front uncertain HV 3. This accomplishment underscores the effectiveness of the proposed fail-operational controller, as it ensures that the primary driving task is not significantly compromised. This finding further supports the high task performance for the EV, aligning with our goal of achieving fail-operational control while upholding travel efficiency.

%%%%%%%%%%%%%%%%%%%%%%%%%
\section{Conclusions}\label{sec:conclusion}
This paper proposes a real-time fail-operational controller for autonomous driving systems, which can adapt to changing environmental disturbances while adhering to state and input constraints. This controller integrates incremental Bayesian learning and control theory, enabling the EV to achieve its desired performance while remaining adaptable to environmental disturbances. 
Our simulation results on a CCC task have substantiated the efficacy of our fail-operational controller, showcasing its ability to safely guide an unsafe EV to a safe state while sustaining the desired performance.  
This achievement was maintained despite the high-velocity HV exhibiting urgent acceleration and deceleration behaviors, along with road and aerodynamic drag disturbances.
% As part of our future work, the fail-operational controller can be extended to safely interact with uncertain HVs at unsignalized intersections for the EV.
%%%%%%%%%%%%%%%%%%%%%%%%%

\bibliographystyle{IEEEtran}
\bibliography{bibliography} 

\end{document}